\documentclass[11pt, a4paper]{article}
\usepackage[utf8]{inputenc}
\usepackage{subcaption} 
\usepackage{enumitem}
\usepackage{amsmath}
\usepackage{amssymb}
\usepackage{amsthm}
\usepackage{graphicx}
\usepackage{booktabs}
\usepackage{tabularx}
\usepackage{hyperref}
\usepackage{algorithm}
\usepackage{algpseudocode}
\usepackage{caption}
\usepackage{times}
\usepackage{xcolor}
\usepackage{tikz}
\usepackage[preprint]{neurips_2026}
\usetikzlibrary{positioning,arrows.meta,calc}
\newcommand{\mat}[1]{\textbf{#1}}
\hypersetup{
    colorlinks=true,
    linkcolor=blue,
    filecolor=magenta,
    urlcolor=cyan,
}

\newcommand{\eh}[1]{$\ll$\textsf{\color{blue} EH : #1}$\gg$}

\newcommand{\linda}[1]{$\ll$\textsf{\color{magenta} Linda : #1}$\gg$}

\newcommand{\newalg}{{RubiConv}}
\newcommand{\ctversion}{{RubiConv-CooleyTukey}}

\def\M{{\mathcal M}}

\newcommand{\ignore}[1]{}

\def\bold0{\mathbf{0}}

\def\epsilon{\varepsilon}

\newtheorem{theorem}{Theorem}[section]

\newtheorem{corollary}[theorem]{Corollary}

\newcommand{\newreptheorem}[2]{%
\newenvironment{rep#1}[1]{%
 \def\rep@title{#2 \ref{##1}}%
 \begin{rep@theorem}}%
 {\end{rep@theorem}}}

\newreptheorem{theorem}{Theorem}
\newreptheorem{lemma}{Lemma}
\newreptheorem{proposition}{Proposition}
\newreptheorem{claim}{Claim}
\newreptheorem{corollary}{Corollary}
\newreptheorem{mainlemma}{Main Lemma}

\newcommand{\namedref}[2]{\mbox{\hyperref[#2]{#1~\ref*{#2}}}}

\newcommand{\figurerefb}[2]{\mbox{\hyperref[#1]{Figure~\ref*{#1}#2}}}

\newcommand{\equationref}[1]{\mbox{\hyperref[#1]{(\ref*{#1})}}}
\renewcommand{\eqref}{\equationref}

\numberwithin{equation}{section}

\title{\textbf{{\newalg} - Efficient Boundary-Respecting Convolutions}}
\date{\today}
\author{Linda Friso \thanks{joint first author} \thanks{Google Deepmind} \And
Annie Marsden$^*$ $^\dagger$ 
\And
Xinyi Chen $^\dagger$ \And
Arushi Gupta $^\dagger$ \And
Peter L. Bartlett $^\dagger$ \thanks{UC Berkeley} \And
Mark Braverman $^\dagger$ \thanks{Princeton University} \And
Elad Hazan $^\dagger$ $^\mathsection$
}
\begin{document}

\maketitle

\begin{abstract}
Convolutional architectures have emerged as powerful alternatives to Transformers for sequence modeling. The primary advantage is that they offer improved theoretical sequence length complexity by leveraging the Fast Fourier Transform (FFT). However, this theoretical improvement does not always meaningfully land in practice. One critical obstacle is that applying standard FFTs is not amenable to the large-scale training pipeline wherein data is packed from different sources into a single sequence for hardware efficiency. Indeed, standard FFT algorithms are not easily amenable to document packing. Existing workarounds suffer from severe inefficiencies, crippling the practical performance of convolutional architectures. We close this gap with 
{{\newalg}}, a novel algorithm for performing hardware-efficient, boundary-respecting convolutions on packed sequences. Extensive experiments show that {\newalg} achieves significant speedups over both attention and standard FFT-based baselines. This work makes the theoretical efficiency of long convolutional models a practical reality for large-scale, real-world data packing.
\end{abstract}

\section{Introduction}

The landscape of sequence modeling has been enriched by a new class of architectures based on long convolutions. Models like Hyena \cite{poli2023hyena}, M2-BERT \cite{fu2023monarch, saadfalcon2024m2bert}, FlashSTU \cite{liu2024flashstu}, and foundational State Space Models (SSMs) like S4 \cite{gu2022s4} achieve strong long-context reasoning abilities across domains from language to genomics. Their computational engine is the Fast Fourier Transform (FFT) convolution, which provides a favorable $O(N \log N)$ complexity \cite{cooley1965algorithm} compared to the $O(N^2)$ cost of attention in Transformers.

In practice, however, a major obstacle arises in large-scale LLM training pipelines: documents of highly variable length are routinely concatenated or ``packed'' together into fixed-size sequences to maximize GPU/TPU utilization. Current attention-based architectures handle packing via masking, and recent works like PackMamba \cite{packmamba2024} have adapted selective, scan-based SSMs \cite{gu2023mamba, lenz2025jamba, lahoti2026mamba3} to this setting by dynamically resetting hidden states. However, convolutional-based architectures cannot account for document packing in the same way. FFT-based algorithms to compute convolutions are hampered.
Continuous convolutions inherently suffer from boundary bleed-over between adjacent documents. Furthermore, FFT-based algorithms compound this by treating the sequence as periodic, introducing global wrap-around artifacts where the sequence's end corrupts its beginning.
Thus, despite their asymptotic advantage, standard FFT convolutions are fundamentally incompatible with packed sequences, creating a gap between theory and practice. Current solutions are inefficient compromises, such as processing documents iteratively (sacrificing parallelism) or padding all documents to a uniform length (wasting memory and computation).

Another solution is to ignore document boundaries completely. However, the resulting data corruption can degrade model quality by introducing a training-dependent artifact. Furthermore, in settings where documents come from different sources, this uncontrolled information leakage represents a critical privacy and security vulnerability.

This paper introduces \newalg, a novel and highly efficient algorithm for performing boundary-respecting convolutions on packed sequences. Our approach is inspired by Bailey's FFT algorithm \cite{bailey1990ffts}, which is naturally amenable to run quickly on modern hardware accelerators. However, Bailey's FFT cannot simultaneously compute multiple FFTs of various packed sequences. Instead, {\newalg{}}'s key innovation is to ensure modular computation across multiple documents simultaneously by changing the permutation mapping and computing different $k$-point DFTs using a block-diagonal matrix structure which depends on each document's sequence length. This preserves document boundaries while leveraging highly optimized GEMM kernels on GPUs and TPUs.

To provide a comprehensive analysis of the problem, we also develop and analyze two alternative methods that illustrate the trade-offs in different computational regimes. First, a \textbf{Full-Matrix} variant provides an exact $O(N^2)$ baseline that is effective for very short sequences. Second, a \textbf{Cooley-Tukey} variant uses arithmetic masking to achieve optimal $O(N \log N)$ asymptotic complexity, demonstrating a path to theoretical efficiency for extremely long sequences. We discuss these alternatives in Section~\ref{sec:alternatives}. While they offer valuable theoretical insight, our experiments show that \newalg, with its $\Theta(N^{3/2})$ complexity but superior hardware efficiency, is the fastest and most practical solution for typical sequence lengths found in large-scale model training. \newalg{} makes the theoretical promise of long convolutions a practical reality for real-world data packing, achieving significant speedups over both standard FFTs and attention.

\subsection{Related Work}


The need to process packed sequences of variable-length documents is a practical constraint in large-scale model training, designed to maximize hardware utilization. While packing optimizations exist for attention models and selective SSMs (e.g., state-resetting in PackMamba \cite{packmamba2024}), standard FFT-based convolutions remain fundamentally incompatible with this data format. Their circular nature, where the convolution is computed modulo the sequence length, causes ``wrap-around" artifacts that incorrectly mix information between independent documents packed into a single tensor. Our work directly addresses this gap between the theoretical efficiency of long convolutions and their practical deployment.

\paragraph{Approaches to Fast Convolution.}
Several algorithms have been developed to accelerate convolutions. The most direct method expresses the convolution as a multiplication with a Toeplitz matrix~\cite{gray2006toeplitz}, but this approach is impractical for long sequences due to its quadratic $O(N^2)$ complexity. The foundational algorithm for fast convolution is the Cooley-Tukey FFT, a divide-and-conquer method that achieves the well-known $O(N \log N)$ complexity by recursively breaking the transform into smaller computations known as "butterfly" operations. As we show in Section~\ref{sec:alternatives}, this algorithm can be adapted to the packed setting via arithmetic masking.

\paragraph{Matrix-Decomposition FFTs.}
A highly influential line of work accelerates FFTs by reformulating them as a series of dense matrix multiplications (GEMMs), making them well-suited for modern hardware. \textbf{Bailey's FFT}~\cite{bailey1989ffts} is a classic example, using a 4-step process that reshapes the input vector into a 2D matrix to improve performance on systems with hierarchical memory. More recently, \textbf{FlashFFTConv}~\cite{fu2023flashfftconv} maps this same underlying arithmetic to modern GPUs, fusing the operations in SRAM to avoid costly memory transfers. 

\paragraph{Our Contribution in Context.}
The primary limitation of prior matrix-decomposition methods, including both Bailey's FFT and FlashFFTConv, is that they are designed and optimized for a single, contiguous sequence. Applying them to packed data would require iterating through each document individually, which sacrifices parallelism and negates their performance advantages. \textbf{\newalg{}} takes a fundamentally different approach. Instead of iterating, it modifies the core Bailey's algorithm by introducing new left and right matrix operators that are aware of the packed structure. By using a block-diagonal matrix for the second DFT, it correctly computes the convolution for all documents in a single, parallel pass, thus solving the long-standing challenge of applying these efficient algorithms to packed sequences.

\section{\newalg: An Accelerated Boundary-Respecting Convolution}\label{sec:new_alg}

\newalg{} is a novel algorithm for performing boundary-respecting convolutions on packed sequences that is specifically designed for modern hardware accelerators. The method is inspired by Bailey's 4-step FFT, which computes an $L$-point DFT by reshaping a 1D vector of length $L =m k$ into an $m \times k$ matrix, applying a first DFT along the columns, element-wise multiplying by a ``twiddle factor'' matrix, and then applying a second DFT along the rows. {\newalg} (Algorithm~\ref{alg:rubiconv}, Figure~\ref{fig:newalg}) proceeds with a similar general structure but with several key differences:

\begin{itemize}[topsep=1.5pt, itemsep=6pt, parsep=0pt]
    \item First, {\newalg} appends $\min(L_i, L_F) - 1$ zeros to each document to enforce linear FFTs, where $L_i$ denotes the length of document $i$ and $L_F$ denotes the filter length. Then it pads the result to ensure its length is divisible by $k$, obtaining a length $L_i'$. Typically picking $k = 256$ provides an optimal tradeoff between padding amount (which is best if $k$ is small) and hardware efficiency (which favors matrix multiplications of specific sizes).
    \item Now that the greatest common divisor of the document lengths is at least $k$, {\newalg} reshapes the sequence into a matrix with $k$ rows so that no single column contains sequences corresponding to different documents. 
    \item The ``twiddle matrix'' is constructed differently than in Bailey's algorithm. In Bailey's, the $(a,b)$-th entry of the twiddle matrix is $\omega_L^{ab}$, where $\omega_L$ denotes the $L$-th root of unity. In {\newalg}, the twiddle matrix is a hybrid of several roots of unity that will ultimately be used to output several different bases for the DFT according to the sequence lengths.
    \item The second DFT adapts to sequences lengths of the documents being packed. The matrix has a block diagonal structure where block $i$ corresponds to the $L_i'/k$-point DFT.
\end{itemize}

\begin{figure}
    \centering
    \includegraphics[width=\linewidth]{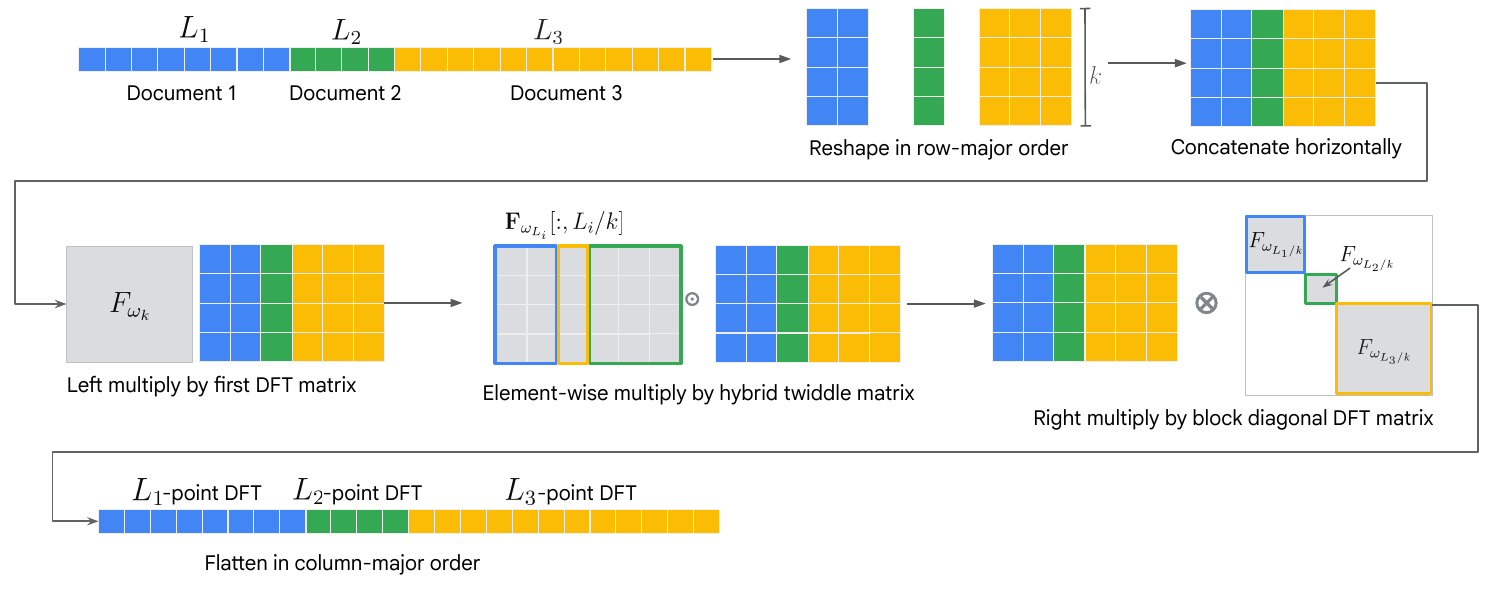}
    \caption{Visual representation of \newalg. }
    \label{fig:newalg}
\end{figure}

\begin{algorithm}[!ht]
\caption{\newalg: Boundary-Respecting Convolution}
\label{alg:rubiconv}
\begin{algorithmic}[1]
\Function{\newalg-Data-Processing}{\text{inputs}, \text{document\_lengths}, $k$}
\State For any integer $j$, let $\mat{F}_{j}$ be the matrix corresponding to the $j$-point DFT, where $$(\mat{F}_{j})_{\ell m} = \exp(2 \pi i \cdot \ell m / j).$$ 
    \State $[ \mat{x}^{(1)}, \dots, \mat{x}^{(n)}] \gets \text{inputs} $
    \State $[L_1, \dots, L_{n}] \gets \text{document\_lengths} $
    \For{$i = 1, \dots, n$}
    \State Append $\min(L_i, L_F) - 1$ zeros to $\mat{x}^{(i)}$ for causality. Then pad to nearest multiple of $k$.
    \State $L_i' \gets k \lceil (L_i + \min(L_i, L_F) - 1) / k \rceil$
    \State $m_i \gets L_i' / k$
    \State $\mat{T}^{(i)} \gets \mat{F}_{{L_i}}[:k, :m_i]$ \Comment{Extract twiddle factors}
    \State Construct and store first DFT matrix: $\mat{M}_1 \gets \mat{F}_{k}$
    \State Construct and store twiddle matrix: $\mat{T} \gets [ \mat{T}^{(1)}, \dots, \mat{T}^{(n)} ]$.
    \State Construct and store second DFT matrix: $$\mat{M}_2 \gets \texttt{BlockDiagonal}[\mat{F}_{{m_1}}, \dots, \mat{F}_{{m_n}}]$$
    \State $\mat{X}^{(i)} \gets $ reshaped $\mat{x}^{(i)}$ with shape $k \times m_i$ in row-major order
    \State $\mat{P}_1$ is the reshaping map, $$\mat{P}_1([ \mat{x}^{(1)}, \dots, \mat{x}^{(n)}])= \begin{bmatrix} \mat{X}^{(1)}  &  \mat{X}^{(2)} & \dots & \mat{X}^{(n)} \end{bmatrix}$$ 
    \State Construct output permutation $\mat{P}_2$ map such that for a $k \times (m_1 + \dots + m_{n})$ matrix $\mat{X} = \begin{bmatrix} \mat{X}_1 & \mat{X}_2 & \dots \mat{X}_{n} \end{bmatrix}$, $\left( \mat{P}_2 \mat{X} \right)= \begin{bmatrix}\text{flat}(\mat{X}_1), \dots, \text{flat}(\mat{X}_{n}) \end{bmatrix}$ where $\text{flat}(\mat{X})$ flattens the matrix in column-major order and slices to original document lengths.
    \EndFor
    \State \Return $\mat{T}$, $\mat{M}_1$, $\mat{M}_2$, $\mat{P}_1$, $\mat{P}_2$    \EndFunction
\Function{\newalg}{\text{inputs}, \text{document\_lengths}, $k$}
\State Reshape the data: $\mat{Y} \gets \mat{P}_1(\text{inputs})$
\State Apply the first DFT: $\mat{Y}' \gets \mat{M}_1 \mat{Y}$
\State Apply twiddle factors element-wise: $\mat{Y}'' \gets \mat{T} \odot \mat{Y}'$
\State Apply the second block-diagonal DFT matrix: $\mat{Y}''' \gets \mat{Y}'' \mat{M}_2$
\State $\text{out} \gets \mat{P}_2(\mat{Y}''')$
    \State \Return $\text{out}$
\EndFunction
\end{algorithmic}
\end{algorithm}

Following the inverse transform, truncating the output to the original valid length of each document removes this padding, ensuring a strictly causal receptive field without intra-document wrap-around artifacts. Theorem~\ref{thm:rubiconv_guarantee} guarantees the correctness of \newalg{} and provides a computational complexity analysis. The proof of Theorem~\ref{thm:rubiconv_guarantee} is in Appendix~\ref{append:rubiconv_proof}. To make Algorithm \ref{alg:rubiconv} performant on a TPU we included several important heuristics and hardware-aware modifications that are detailed in Appendix \ref{sec:rubi-optimizations}.

\ignore{
\begin{theorem}
\label{thm:rubiconv_guarantee}
    The {\newalg} Algorithm as described in Alg.~\ref{alg:rubiconv} when applied to a packed sequence $[\mat{x}^{(1)}, \mat{x}^{(2)}, \dots, \mat{x}^{(n)}]$ outputs the packed sequence $[\overline{\mat{x}}^{(1)}, \overline{\mat{x}}^{(2)}, \dots,\overline{\mat{x}}^{(n)}]$ such that for each $i = 1, \dots, n$, $\overline{\mat{x}}^{(i)}$ is the $L'_i$-point DFT of $\mat{x}^{(i)}$ where $L'_i = k \times \lceil L_i/k \rceil$. Moreover, {\newalg}-data preprocessing has a computational complexity and memory footprint of $O\left( L_{\text{total}} + kn + (L_{\text{total}}/k)^2 + n^2 + L_{\text{total}} n/k  \right)$ and {\newalg} has computational complexity $O(\left(k L_{\text{total}} + L_{\text{total}}^2/k) + k^2 n + n^2 k + L_{\text{total}} n \right)$. 
\end{theorem}
}

\begin{theorem}
\label{thm:rubiconv_guarantee}
    The {\newalg} Algorithm as described in Alg.~\ref{alg:rubiconv}, when applied
    to a packed sequence $[\mat{x}^{(1)}, \mat{x}^{(2)}, \dots, \mat{x}^{(n)}]$
    with original document lengths $L_1, \dots, L_n$ and filter length $L_F$,
    outputs the packed sequence
    $[\overline{\mat{x}}^{(1)}, \overline{\mat{x}}^{(2)}, \dots, \overline{\mat{x}}^{(n)}]$
    such that for each $i = 1, \dots, n$, $\overline{\mat{x}}^{(i)}$ is the
    $L_i'$-point DFT of $\mat{x}^{(i)}$, where
    $L_i' = k \lceil (L_i + \min(L_i, L_F) - 1) / k \rceil$.
    Moreover, {\newalg}-data preprocessing has a computational complexity and
    memory footprint of
    $O\left( L_{\text{total}} + kn + (L_{\text{total}}/k)^2 + n^2
    + L_{\text{total}} n/k \right)$
    and {\newalg} has computational complexity
    $O\left(k L_{\text{total}} + L_{\text{total}}^2/k + k^2 n + n^2 k
    + L_{\text{total}} n \right)$.
\end{theorem}

Corollary~\ref{cor:like_bailey} shows that, for a certain regime, \newalg{}'s complexity matches Bailey's algorithm as if no documents were packed together. 
\begin{corollary}[Corollary to Theorem~\ref{thm:rubiconv_guarantee}]
\label{cor:like_bailey}
If the number of packed documents satisfies $n = \Theta \left(  L_{\text{total}}^{1/2} \right)$ and we pick $k = \Theta(L_{\text{total}}^{1/2})$ the preprocessing step has complexity $O(L_{\text{total}})$ and {\newalg} has an overall complexity of $O(L_{\text{total}}^{3/2})$.
\end{corollary}

\section{Alternative Boundary-Preserving Formulations} \label{sec:alternatives}

While \newalg~represents our primary contribution and the most practical solution for modern hardware, the core principle of boundary preservation can be extended to other convolution algorithms. In this section, we present two alternative formulations to provide a comprehensive analysis of the problem. While these methods are not empirically competitive with \newalg~for typical large-scale deep learning workloads on accelerators, they map the bounds of the solution space and illustrate the critical trade-offs between asymptotic complexity and hardware utilization.

\subsection{An Asymptotically Optimal Variant: \ctversion}

It is theoretically possible to achieve the optimal asymptotic complexity of $O(N \log N)$ by adapting the classic iterative Cooley-Tukey FFT algorithm. Our variant, \ctversion, achieves this not by altering the FFT's overarching structure, but by \textbf{arithmetically masking} butterfly operations iteratively to prevent data from crossing document boundaries.

The core idea is to proceed through the standard Cooley-Tukey stages, where each stage handles a particular DFT size, and generate twiddle factors based on each element's parent document length. At stage $i$, the DFT size is $m=2^i$, and if a document is too short to participate in the current DFT size (i.e., its length is less than $m$), its twiddle factors are set to identity values. This effectively turns the butterfly update into a no-op for that document, preserving its state from the last valid stage. 

At a high level, the algorithm opens up the Cooley-Tukey stages and enforces document boundaries by masking in each stage, thus maintaining the asymptotic runtime of the FFT. The algorithm (detailed in Alg. \ref{alg:masked_fft}) begins with a per-document bit-reversal and proceeds with these masked butterfly stages across the packed tensor. We include a visualization of the algorithm in Appendix~\ref{sec:app_alternatives}.

\begin{algorithm}[H]
\caption{\ctversion: Boundary-Respecting Cooley-Tukey}
\label{alg:masked_fft}
\begin{algorithmic}[1]
\Function{\text{\ctversion}}{\texttt{input\_array}, \texttt{pattern}}
    \State \Comment{Assume inputs are padded such that each document length is a power of 2}
    \State Let \texttt{y} be a copy of \texttt{input\_array}.
    \ForAll{document\_block in \texttt{pattern}}
        \State Reorder elements within \texttt{y} for this block into bit-reversed order.
    \EndFor
    \State $\textit{max\_log2\_len} \gets \text{maximum log2-size from } \texttt{pattern}$
    \For{$s \gets 1$ to $\textit{max\_log2\_len}$}
        \State $m \gets 2^s$  
        \ForAll{element at index $i$ in \texttt{y}}
            \State $\textit{doc\_len} \gets \text{length of document containing element } i$
            \If{$\textit{doc\_len} \ge m$}
                \State $(\texttt{Tw0}[i], \texttt{TwF}[i], \texttt{TwB}[i]) \gets \text{CalculateStandardFFTFactors}(i, s)$
            \Else
                \State $(\texttt{Tw0}[i], \texttt{TwF}[i], \texttt{TwB}[i]) \gets (1, 0, 0)$  
            \EndIf
        \EndFor
        \State $y \gets (y \odot \text{Tw0}) + (\text{roll}(y, m/2) \odot \text{TwF}) + (\text{roll}(y, -m/2) \odot \text{TwB})$
    \EndFor
    \State \Return{\texttt{y}}
\EndFunction
\end{algorithmic}
\end{algorithm}

\paragraph{Correctness, Complexity, and Hardware Limitations}
This method correctly computes the concatenation of independent, per-document FFTs. An inductive proof shows that at every stage $s$, the values within each document's slice are identical to those from a standard Cooley-Tukey FFT run on that document. However, because an active node will still read across a boundary even if the write is masked, preventing information mixing requires padding each document to the nearest power of two. This alignment ensures cross-document butterfly pairings never occur. While preserving the $O(N \log N)$ complexity for the padded sequence, it introduces severe memory bloat.

Crucially, while \ctversion~is theoretically optimal, it is fundamentally misaligned with modern accelerator architectures like TPUs and GPUs. These devices achieve peak, high-bandwidth performance through large, dense matrix multiplications (GEMMs). The iterative, memory-bound butterfly updates required by the CT algorithm result in poor hardware utilization. This architectural bottleneck, compounded by the power-of-two padding requirement, renders the $O(N \log N)$ approach impractical for large-scale model training.

\subsection{A Direct Quadratic-Time Variant via Block-Diagonal Matrices}

The most straightforward way to ensure boundary-respecting convolution is to represent the convolution operator as an explicit \textbf{block-diagonal matrix}. For a packed vector $x = [x^{(1)} \,\|\, \cdots \,\|\, x^{(B)}]$, we construct a large matrix $A = \mathrm{diag}(A_1, \ldots, A_B)$, where each block $A_b$ is the matrix operator for the convolution on document $b$. Depending on the desired padding, $A_b$ can be a circulant matrix (for circular convolution) or a Toeplitz matrix (for linear convolution).

The convolution is then a single matrix-vector product, $y = Ax$. By its block-diagonal construction, the operator mathematically guarantees that it cannot mix information between documents, as each block $A_b$ only acts on the corresponding input segment $x^{(b)}$.

\paragraph{Complexity.}
Forming and applying this matrix costs $\Theta(\sum_b L_b^2)$, which is upper-bounded by $\Theta(N^2)$. Although asymptotically inefficient compared to FFT-based methods, this approach serves as a simple, exact, and brute-force baseline. It preserves documents boundaries and maps cleanly to GEMM kernels, making it highly competitive for processing extremely short sequences where the overhead of FFTs is unjustified, or for validating the numerical outputs of more complex algorithms.

\subsection{Synthesizing the Solution Space}
The development of these alternatives clearly bounds the design space for packed sequence convolutions. The Full-Matrix approach ($O(N^2)$) maps perfectly to hardware GEMMs but is too computationally expensive for long contexts. Conversely, the \ctversion~approach ($O(N \log N)$) is algorithmically optimal but hostile to modern accelerator memory hierarchies. Therefore, \newalg---with its $\Theta(N^{3/2})$ complexity---emerges as the most pragmatic solution. It is the only method that successfully marries theoretical boundary preservation with the dense GEMM operations that modern AI accelerators demand.

\section{Experimental Evaluation} \label{sec:experimental_evaluation}
\subsection{Runtime and Efficiency} \label{sec:runtime}
In this section we analyze the efficiency of \newalg{} by comparing it to other convolution baselines as well as splash attention.
\subsubsection{Experimental Setup}
\paragraph{Baselines}
We analyze the scaling behavior of \newalg{} execution time with respect to sequence length, model dimension, and convolution filter size, comparing it against the following baselines:
\begin{enumerate}
    \item \textbf{\texttt{jnp.convolve}.} This baseline processes a packed sequence by iterating through each document. For a given document, it masks all tokens belonging to other documents and applies a causal convolution using the JAX NumPy function \texttt{jnp.convolve}. This approach correctly handles document boundaries but incurs the cost of performing a sequential loop over the documents in the packed sequence.
    \item \textbf{\texttt{jax.lax.conv}.} This method is identical to the \texttt{jnp.convolve} baseline but substitutes the convolution operation with JAX's optimized \texttt{jax.lax.conv\_general\_dilated} function.
    \item \textbf{Splash Attention.} As a hardware performance baseline, we compare against Splash Attention, a highly optimized Pallas software kernel for TPUs (OpenXLA Contributors, 2024) that maintains standard attention computations within on-chip SRAM. Although architecturally distinct from a convolution, this comparison contextualizes the runtime of \newalg{} against one of the fastest and most widely adopted layers in large-scale language models\footnote{If the ultimate goal of \newalg{} is to unlock convolutional architecture blocks that truly serve as an alternative to the transformer's global attention block, then we would like to ensure its efficiency is comparable or better.}.
\end{enumerate}
\paragraph{Benchmarking protocol}
Unless otherwise specified, all runtime measurements for convolution algorithms evaluate a single convolution operation within a single network layer, mirroring a standard language model training setup. The input is a tensor $\mathbf{X} \in \mathbb{R}^{ L \times D}$ (where $L$ is the sequence length and $D$ is the model's dimension), and the convolution filter is represented as $\mathbf{F} \in \mathbb{R}^{L_F \times D}$ (where $L_F$ is the filter length). The convolution is applied independently to each of the $D$ feature channels, acting as a depthwise convolution.
All runtime scaling experiments are conducted on a system with $4$ TPU v7 devices. \newalg{} and the baselines are sharded along the sequence dimension. When \newalg{} reshapes the sequence into a 2D matrix (mapping sequence length to rows and columns), the sequence-axis sharding translates to column-axis sharding. Consequently, the algorithm's subsequent row-wise matrix multiplications execute locally, while the column-wise multiplications are automatically sharded across TPUs.
To obtain stable and reliable measurements, we report the mean runtime over $5$ trials, preceded by $5$ untimed warm-up iterations, and 95\% confidence intervals. For each configuration of batch size, sequence length, and model dimension, identical input data was used across all algorithms to ensure a fair comparison.
\paragraph{Data}
To ensure our scaling experiments are representative of real-world scenarios, we established an empirical document length distribution. We derived this distribution by tokenizing a 10B-token sample of the 1.3T-token FineWeb-Edu dataset~\cite{penedo2024fineweb} using the Gemma 3~\cite{gemmateam2025gemma3technicalreport} tokenizer. While the token values in our experimental dataset are synthetic, the document boundaries are determined by sampling document lengths from this empirical distribution.

\subsubsection{Results}

Algorithm~\ref{alg:rubiconv} decomposes \newalg{} into data pre-processing and an FFT component. As detailed in Appendix~\ref{sec:rubi-optimizations}, pre-processing--- such as computing document-based permutations and DFT matrices--- is dependent only on the current segment ids. In language models, this overhead amortizes efficiently across multiple convolutional layers. Hence we report two benchmarks: \newalg{} (Total) for the full pipeline, and \newalg{} (Conv Only) for the isolated FFT execution.

\paragraph{Sequence length scaling}
We measure the scaling of different convolution algorithms with respect to sequence length up to $2^{18}$, keeping the model dimension constant and the convolution filter size equal to the sequence length. \newalg{} performance is measured with documents padded to a multiple of $k=256$. Figure~\ref{fig:all_sequence_scaling} shows that \newalg{} exhibits superior asymptotic scaling compared to iterative convolution baselines. While brute-force methods are competitive for extremely short sequences, \newalg{} quickly overtakes them, consistently outperforming all competing convolutional methods across standard context lengths. On the single dimensional example we see that for context lengths exceeding $2^{16}$, \newalg{} is faster than Splash Attention, empirically demonstrating the feasibility of our method in practical, large-scale training regimes. This advantage doesn't appear with higher model dimension (1024), however for long enough sequences (sequence length 260k) \newalg{} is competitive with splash attention.

\begin{figure}[!ht]
\centering
\begin{subfigure}[b]{0.43\textwidth}
    \centering
 \includegraphics[width=\linewidth]{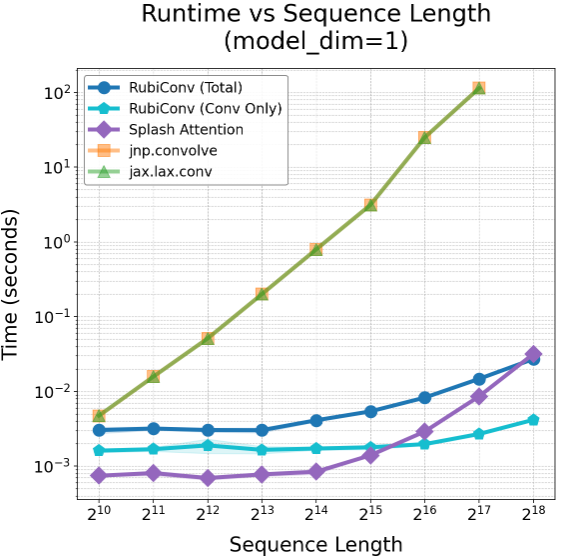} 
    \caption{Scaling with sequence length ($L$), with filter size $L_F = L$ and model dimension $D = 1$.}
    \label{fig:seq_len_scaling_hdim1}
\end{subfigure}
\hfill
\begin{subfigure}[b]{0.43\textwidth}
    \centering
 \includegraphics[width=\linewidth]{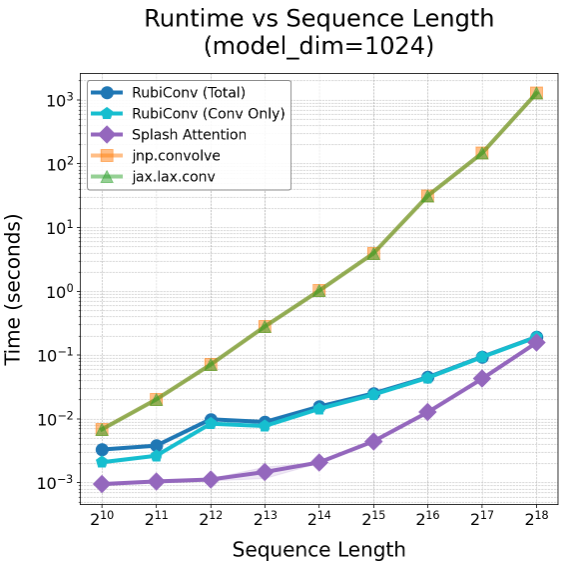} 
    \caption{Scaling with sequence length ($L$), with filter size $L_F = L$, model dimension $D = 1024$.}
    \label{fig:seq_len_scaling_hdim1024}
\end{subfigure}
\caption{Scaling of different convolution algorithms with respect to sequence length.}
\label{fig:all_sequence_scaling}
\end{figure}

\paragraph{Model dimension scaling}
We measure the scaling of different convolution algorithms with respect to the model dimension $D$, keeping the sequence length and convolution filter size equal to $2^{14}$. \newalg{} performance is measured with documents padded to a multiple of $k=256$, and Splash Attention performance is measured with the head dimension equal to $\min(256, D)$. Figure~\ref{fig:hidden_dim_scaling} illustrates that \newalg{} scales efficiently with the model dimension, maintaining a performance advantage over the convolution baselines.

\paragraph{Convolution filter size scaling}
We measure the scaling of different convolution algorithms with respect to the convolution filter size, keeping the sequence length equal to $2^{14}$ and the model dimension equal to $1024$. \newalg{} performance is measured with documents padded to a multiple of $k=256$. Figure~\ref{fig:conv_filter_scaling} shows that, unlike other convolutional baselines, \newalg{} runtime remains computationally tractable with long filter sizes, making it suitable for extremely long convolutions.

\begin{figure}[h!]
\centering
\begin{subfigure}[b]{0.43\textwidth}
    \centering
 \includegraphics[width=\linewidth]{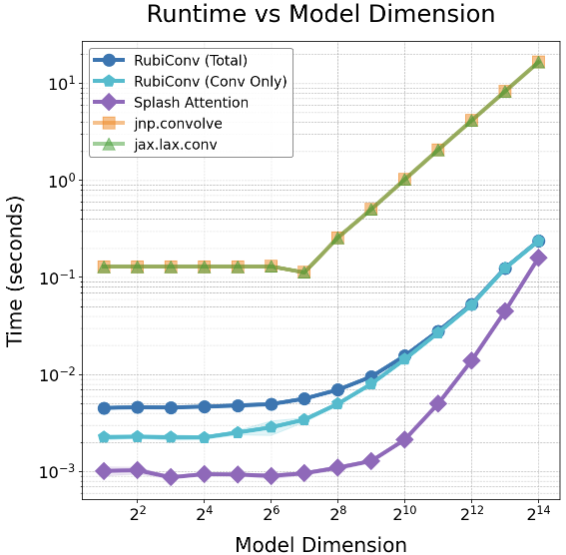} 
    \caption{Scaling with model dimension ($D$) with sequence length $L = 2^{14}$, filter size $L_F = 2^{14}$.}
    \label{fig:hidden_dim_scaling}
\end{subfigure}
\hfill
\begin{subfigure}[b]{0.43\textwidth}
    \centering
 \includegraphics[width=\linewidth]{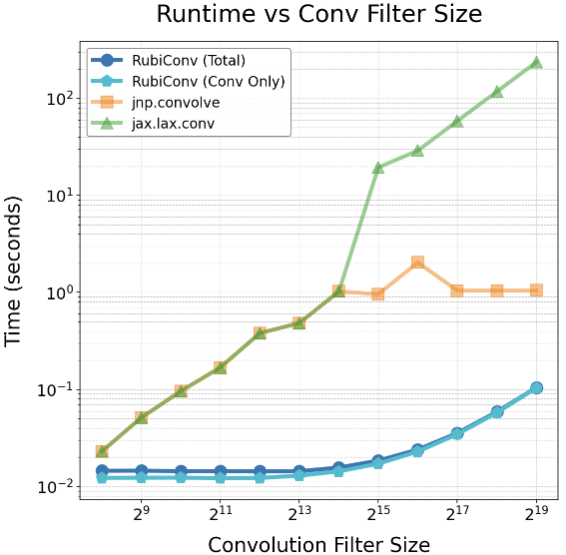} 
    \caption{Scaling with filter size ($F$) with sequence length $L = 2^{14}$, model dimension $D = 1024$.}
    \label{fig:conv_filter_scaling}
\end{subfigure}
\caption{Scaling of different algorithms with respect to convolution filter size and model dimension.}
\label{fig:all_scaling}
\end{figure}

\subsection{The Necessity of Boundary-Aware Convolutions}
\label{sec:synthetic_tasks}

Packing documents into a single sequence for efficient training causes standard convolutions to mix representations across boundaries. Evaluating a hybrid model on synthetic packed-sequence tasks reveals that standard convolutions conflate independent contexts, dropping accuracy to near-chance at high packing densities. In contrast, boundary-respecting convolutions sustain high accuracy, proving that strict boundary preservation is a functional prerequisite for reliable reasoning.

\paragraph{Tasks and Training}
We evaluate two synthetic tasks (Table~\ref{tab:task_params}): \textbf{Noisy Recall}~\cite{poli2024mechanistic} and \textbf{Associative Retrieval}~\cite{ba2016using}. Models train on sequences of 24 dynamically generated, concatenated documents. For Noisy Recall the data consists of a target symbol and a variable number of distractors, e.g. $[v,\; n_1 \ldots n_k,\; \texttt{ASK},\; v]$ and the task is to recall the target symbol $v$. Noisy Recall tightly packs variable-length documents, right-padding to a maximum length. Associative Retrieval uses fixed-length documents with entity--attribute pairs, e.g. $[N_1\!\texttt{ likes }A_1\texttt{.}\;\ldots\;\texttt{What likes }N_q\texttt{?}\;\to A_q]$, sampled without replacement to prevent contradictions; query targets are chosen uniformly at random. To enable boundary-aware masking, models receive per-position document IDs. Cross-entropy loss is computed exclusively at the \emph{answer position} per document and AdamW optimizer is used ($\beta_1{=}0.9$, $\beta_2{=}0.95$, weight decay $0.01$, LR $3{\times}10^{-3}$). Experiments were executed on a single Google Cloud TPU v2. Performance is evaluated via \emph{answer accuracy} (argmax prediction matching the target) over 8 test batches, reporting averages and 95\% confidence intervals across 5 seeds.

\paragraph{Model Architecture}
We employ a lightweight hybrid architecture interleaving standard multi-head causal self-attention blocks with learned long-convolution layers. For the convolutions, we use learned depthwise causal filters of shape $(T, d)$, where $T$ is the sequence length and $d$ is the model width. We use sequence-length filters to evaluate the convolution's capacity for long-range dependency modeling. Token and positional embeddings are added at the input, and a final layer norm and linear head project to next-token logits.
For both tasks, we train a 2-layer model with model dimension $d=64$, $4$ attention heads and training batch size of $4$.

\begin{table}[h!]
\centering
\resizebox{\textwidth}{!}{%
\begin{tabular}{lcc}
\toprule
 & \textsc{Noisy Recall} & \textsc{Associative Retrieval} \\
\midrule
\textbf{Document format} & $[v,\; n_1 \ldots n_k,\; \texttt{ASK},\; v]$
   & $[N_1\!\texttt{ likes }A_1\texttt{.}\;\ldots\;\texttt{What likes }N_q\texttt{?}\;\to A_q]$ \\
\textbf{Answer token}  & Target symbol $v$  & Attribute name $A_q$ \\
\textbf{Documents per sequence}         & 24  & 24 \\
\textbf{Vocab size}     & 258 (256 symbols + 2 special)  & 15 (5 names + 5 attrs + 5 special) \\
\textbf{Max distractors / facts}  & $k{\le}30$ distractors  & exactly 2 facts per doc \\
\bottomrule
\end{tabular}%
}
\caption{Synthetic tasks configurations.}
\label{tab:task_params}
\end{table}

\paragraph{Convolution Modes}
To isolate the effect of convolution boundary mixing, we ablate two variants while keeping the architecture and standard causal attention fixed. Crucially, attention crosses document boundaries in both settings; therefore, any performance discrepancy is solely attributable to the convolution layer.
We compare a \emph{boundary-respecting} variant, where document-ID metadata masks the kernel to restrict aggregation to intra-document tokens, against a \emph{document-mixing} variant, which applies the kernel causally across the entire packed sequence.


\paragraph{Results}
As illustrated in Figure~\ref{fig:synthetic_tasks}, boundary-respecting convolutions demonstrate rapid and stable convergence, achieving full accuracy well before the end of the training window on both tasks. In contrast, standard document-mixing convolutions severely delay learning, exhibit high variance, and fail to reach comparable accuracy within 20,000 steps.

\begin{figure}[h!]
\centering
\begin{subfigure}[b]{0.4\textwidth}
    \centering
 \includegraphics[width=\linewidth]
 {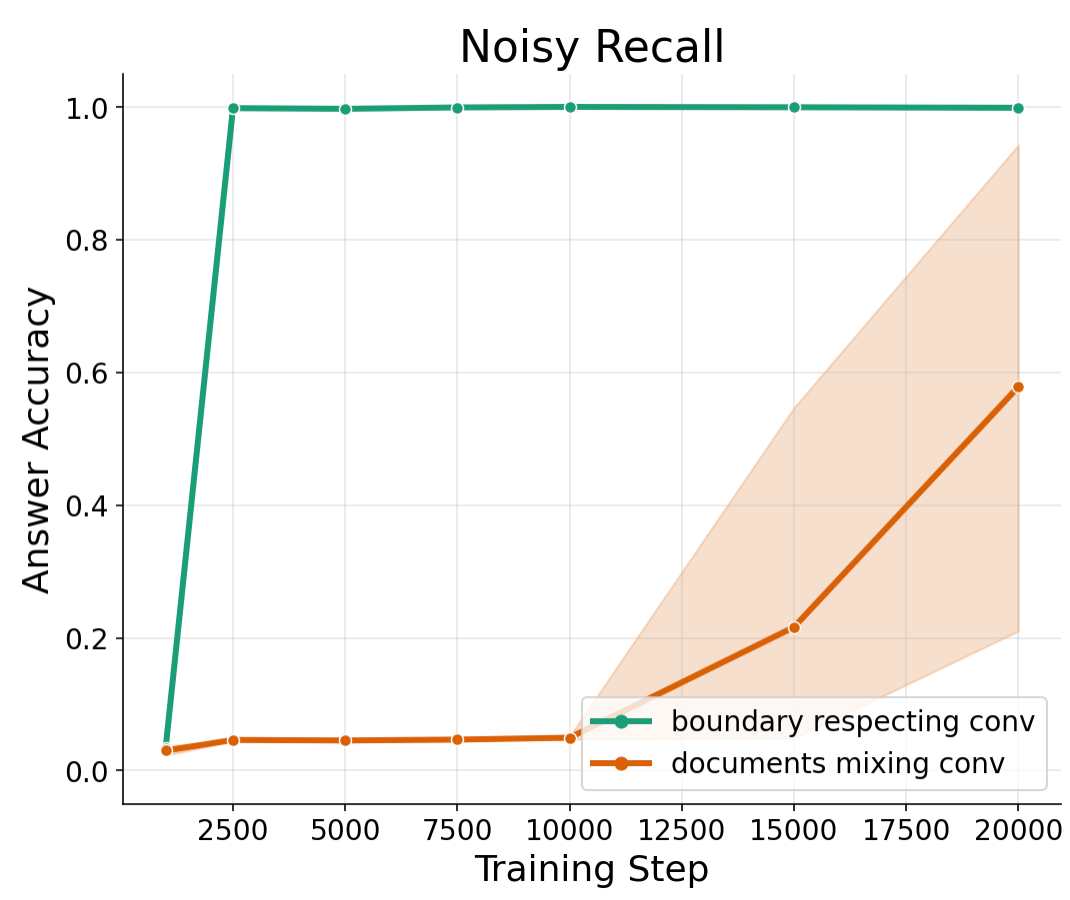} 
    \caption{Noisy Recall accuracy.}
    \label{fig:synthetic_tasks_noisy_recall}
\end{subfigure}
\hfill
\begin{subfigure}[b]{0.4\textwidth}
    \centering
 \includegraphics[width=\linewidth]
 {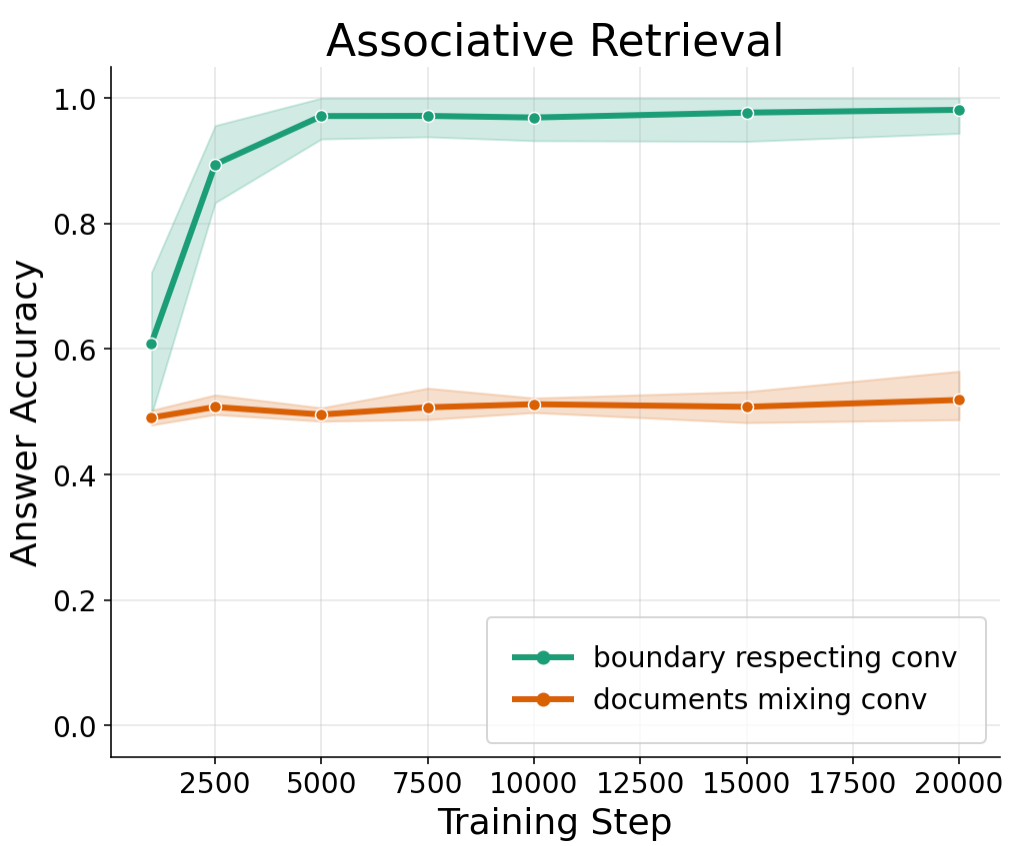} 
    \caption{Associative Retrieval accuracy.}
    \label{fig:synthetic_tasks_associative_lookup}
\end{subfigure}
\caption{Answer accuracy vs training duration comparison for boundary respecting and documents mixing convolutions. We report the average of 5 training runs with 95\% confidence intervals.}
\label{fig:synthetic_tasks}
\end{figure}
\vspace*{-0.22cm}
\section{Conclusion}

We have addressed the critical challenge of applying fast, FFT-based convolutions to the packed sequences commonly used in large-scale training pipelines. While standard FFTs are asymptotically efficient, their circular nature makes them incompatible with this data format, creating a gap between their theoretical promise and practical utility. This paper introduces \textbf{\newalg}, a novel and performant algorithm that definitively solves this problem. By adapting the matrix-decomposition approach of Bailey's FFT, \newalg{} performs correct, boundary-respecting convolutions on packed data without resorting to inefficient iterative processing or wasteful padding.

The core innovation of \newalg{} is its reformulation of the packed convolution problem as a series of parallel matrix multiplications. By structuring the second DFT operator as a \textbf{block-diagonal matrix}, we mathematically guarantee that no information is mixed across document boundaries, all within a single computational pass. This approach maps elegantly to modern hardware accelerators, leveraging highly optimized GEMM kernels to achieve significant speedups over existing baselines and attention mechanisms. While its theoretical complexity is $\Theta(N^{3/2})$, our experiments confirm it is the most practical and efficient solution for the sequence lengths found in real-world workloads.

To provide a complete theoretical analysis, we also presented alternative boundary-preserving formulations, including an asymptotically optimal $O(N \log N)$ Cooley-Tukey variant (\ctversion)  and a direct $O(N^2)$ matrix baseline. By providing a correct, performant, and scalable solution for this ubiquitous data format, \newalg{} bridges the final gap to deploying long convolutional models efficiently at scale. This work paves the way for the wider adoption and continued improvement of a powerful class of architectures for long-sequence modeling.

\newpage 

\bibliographystyle{plain}
\bibliography{main}

\newpage

\newpage

\appendix

\section{Appendix}
\subsection{The Crossover Point of Algorithmic vs Hardware Efficiency}
A central motivation for using FFT-based convolutions is their favorable asymptotic complexity of $O(N \log N)$ compared to the $O(N^2)$ of direct convolution or standard attention mechanisms. However, in practice, this theoretical advantage is mediated by the underlying hardware. Modern accelerators like GPUs and TPUs have highly optimized, low-level kernels for dense matrix-matrix and matrix-vector multiplications, which can outperform algorithmically superior methods at smaller problem sizes.

To empirically ground the need for efficient, long-sequence convolution algorithms, we conducted an experiment to identify the ``crossover point'' where the algorithmic efficiency of the Fast Fourier Transform overtakes the raw performance of a hardware-optimized dense matrix-vector multiplication.

\subsection*{Experimental Setup}

We benchmarked three distinct operations on CPU, GPU, and TPU backends across a range of sequence lengths ($N = 2^8$ to $N = 2^{16}$):

\begin{enumerate}
    \item \textbf{Random Matrix-Vector Multiplication ($O(N^2)$):} A multiplication between a dense, unstructured random matrix and a vector. This measures the baseline performance of a brute-force, hardware-optimized GEMV (General Matrix-Vector) operation.
    
    \item \textbf{DFT Matrix-Vector Multiplication ($O(N^2)$):} A multiplication using an explicitly constructed, dense DFT matrix. This isolates the effect of matrix structure, if any, without leveraging a specialized algorithm.
    
    \item \textbf{FFT Algorithm ($O(N \log N)$):} The standard, highly optimized JAX FFT implementation, which uses an efficient Cooley-Tukey-style algorithm.
\end{enumerate}

\subsection*{Results and Analysis}

The results, shown below, clearly illustrate the trade-off between hardware specialization and algorithmic complexity.

\begin{figure}[h!]
\centering
\begin{subfigure}[b]{0.32\textwidth}
    \centering
 \includegraphics[width=\linewidth]{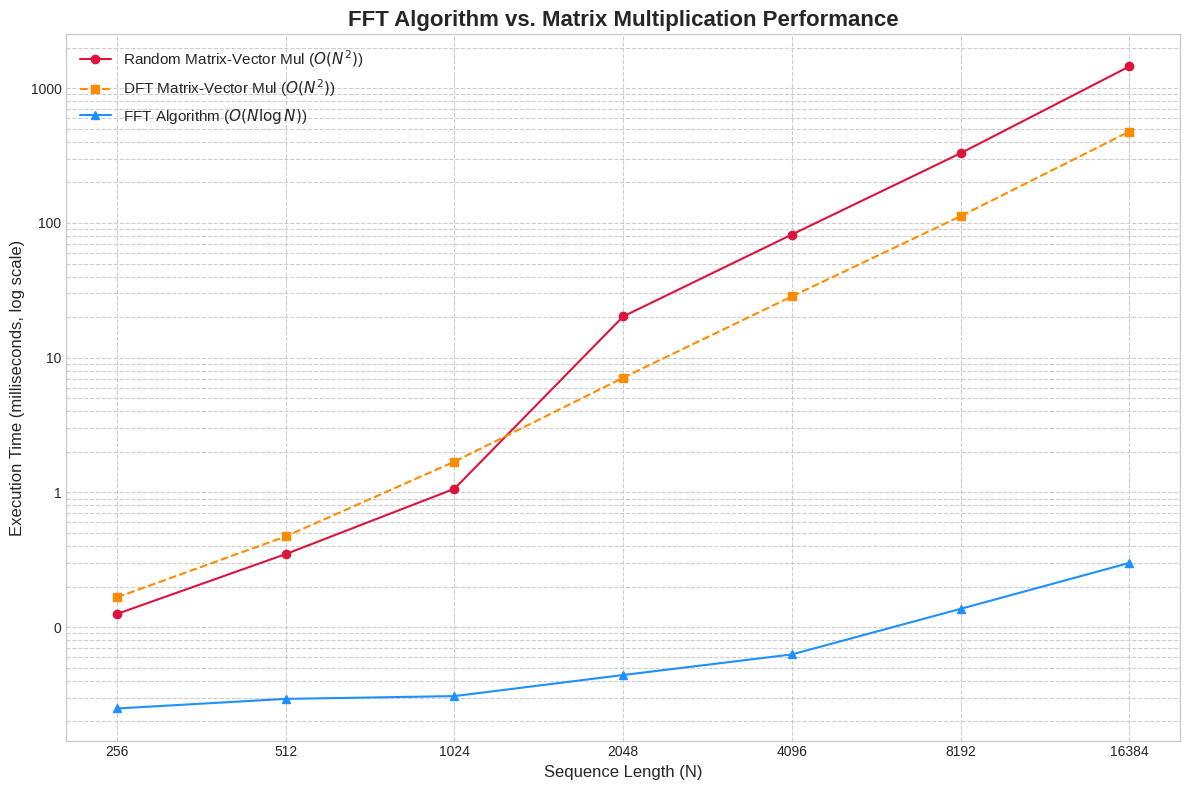} 
    \caption{Performance on CPU. The algorithmic advantage of the FFT is apparent even at relatively small sequence lengths, with a crossover point occurring early.}
    \label{fig:cpu_perf}
\end{subfigure}
\hfill
\begin{subfigure}[b]{0.32\textwidth}
    \centering
 \includegraphics[width=\linewidth]{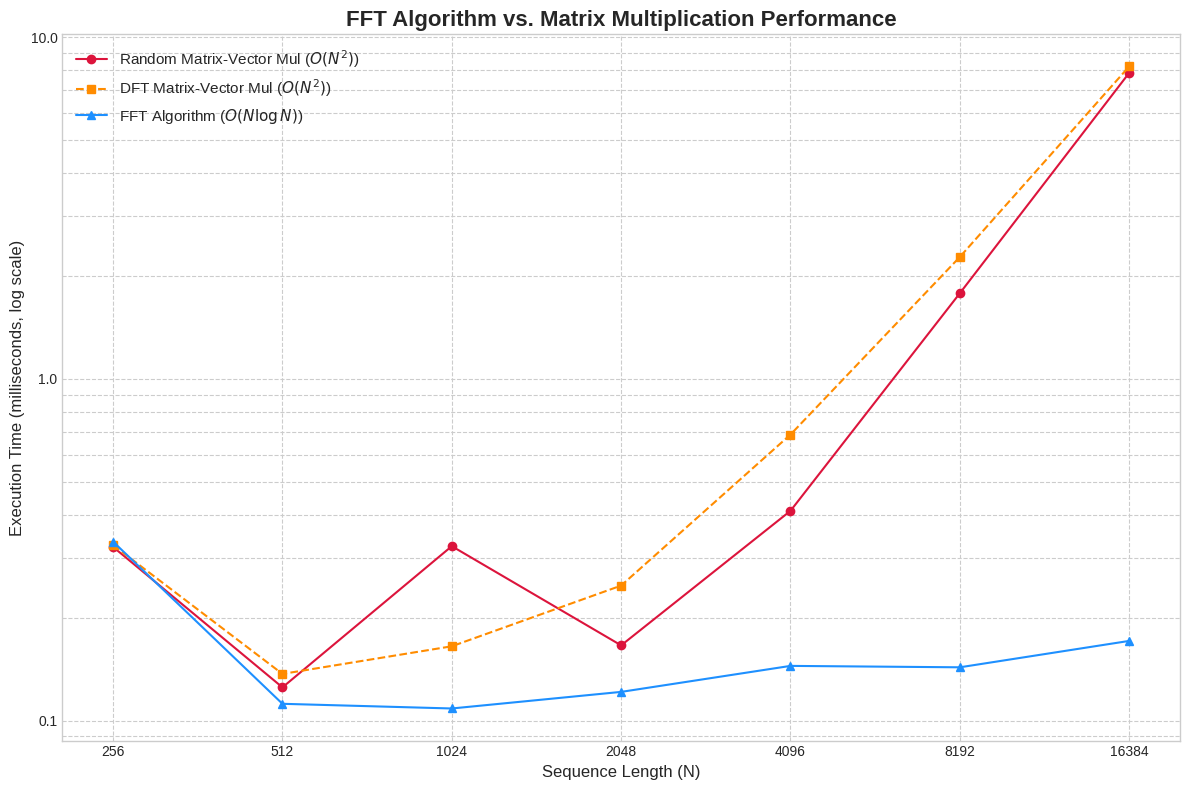} 
    \caption{Performance on GPU (NVIDIA T4). The crossover point is pushed to a larger sequence length due to the GPU's parallelism, but the FFT's superior scaling eventually dominates.}
    \label{fig:gpu_perf}
\end{subfigure}
\hfill
\begin{subfigure}[b]{0.32\textwidth}
    \centering
 \includegraphics[width=\linewidth]{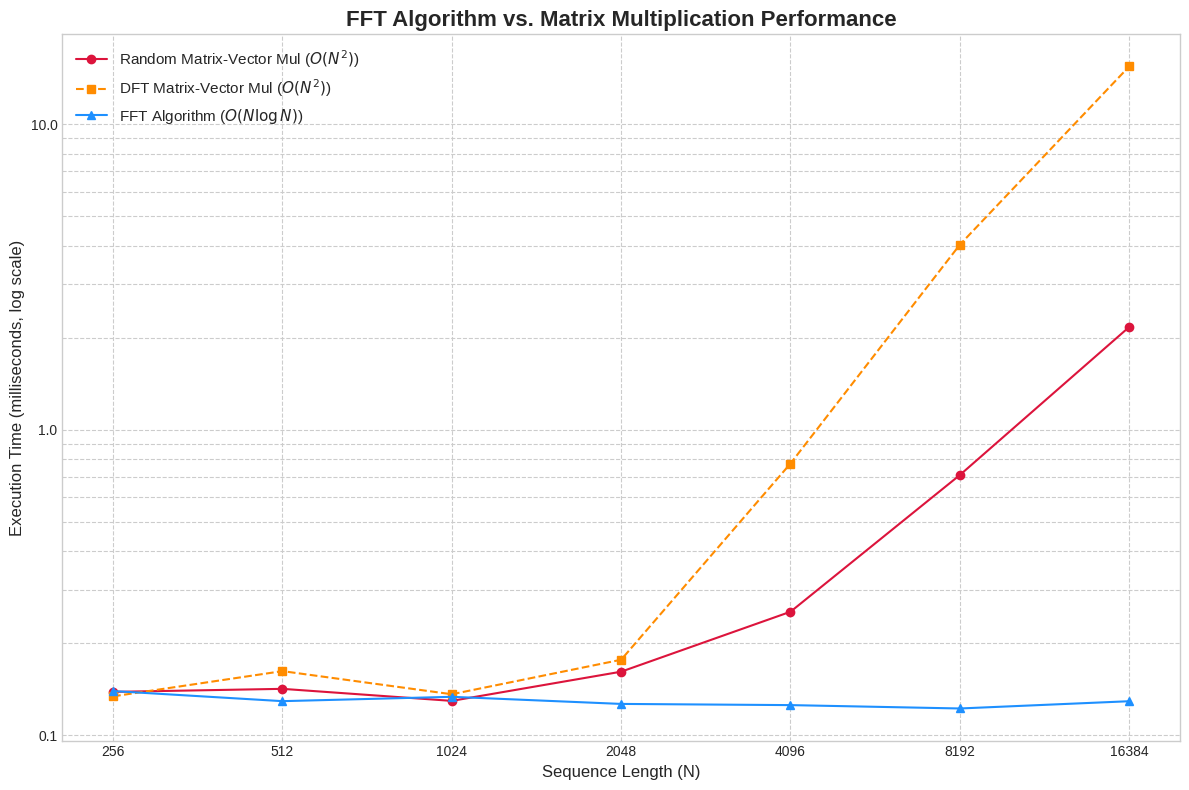} 
    \caption{Performance on TPU (v3). The effect is even more pronounced, with the crossover at a very large sequence length. For modern LLMs, the $O(N \log N)$ FFT is essential.}
    \label{fig:tpu_perf}
\end{subfigure}
\caption{Performance comparison of matrix-vector multiplication methods on different hardware architectures. These plots show the point at which the algorithmic efficiency of the FFT ($O(N \log N)$) overcomes the hardware optimization for dense matrix operations ($O(N^2)$).}
\label{fig:all_perf}
\end{figure}

These results confirm that while hardware optimizations can make brute-force methods competitive for shorter sequences, algorithmic efficiency is paramount for scaling to the long contexts required by state-of-the-art models. This provides a strong motivation for developing methods like \newalg, which make the theoretical efficiency of FFT-based convolutions a practical reality in production settings.

\subsection{Formal {\newalg} Guarantee: Proof of Theorem~\ref{thm:rubiconv_guarantee}}
\label{append:rubiconv_proof} 
\begin{proof}[Proof of Theorem~\ref{thm:rubiconv_guarantee}]
From Algorithm~\ref{alg:rubiconv} we have the following notation and assumptions:
\begin{itemize}
    \item We have $n$ vectors $\mat{x}^{(1)}, \dots, \mat{x}^{(n)}$ each with length $L_i' = k \times m_i$.
    \item $\omega_n$ is the $n$-th root of unity $\exp(2 \pi i/n)$ 
    \item $\mat{F}_k$ is the $k \times k$ matrix with $(a,b)$-th entry $\omega_k^{ab}$.
    \item $\mat{X}^{(i)}$ is the $k \times m_i$ matrix of $\mat{x}^{(i)}$ reshaped in row-major order,
    \begin{equation*}
        X^{(i)} = \begin{bmatrix}
            x^{(i)}_0 & x^{(i)}_1 & \dots & x^{(i)}_{m_i-1} \\
            x^{(i)}_{m_i} & x^{(i)}_{m_i+1} & \dots & x^{(i)}_{2 m_i-1}  \\
            \vdots & && \vdots \\
            x^{(i)}_{(k-1)m_i} &&& x^{(i)}_{km_i - 1}
        \end{bmatrix}
    \end{equation*}
    \item $\mat{T}^{(i)}$ is the $k \times m_i$ twiddle-matrix with $(a,b)$-th entry $\omega_{L_i'}^{ab}$
\end{itemize}
First we show that it suffices to consider the result of a single document since the algorithm is equivalent to an algorithm that handles each separately. Indeed, first it constructs the matrix
\begin{equation*}
    \mat{Y} \gets \begin{bmatrix}
        \mat{X}^{(1)} & \mat{X}^{(2)} & \dots & \mat{X}^{(n)}.
    \end{bmatrix}
\end{equation*}
Then it left multiplies by $\mat{F}_k$,
\begin{equation*}
    \mat{Y}' \gets \mat{F}_k \mat{Y} = \begin{bmatrix}
       \mat{F}_k \mat{X}^{(1)} & \mat{F}_k \mat{X}^{(2)} & \dots & \mat{F}_k \mat{X}^{(n)}.
    \end{bmatrix}
\end{equation*}
Next it applies the twiddle factors through element-wise multiplication. The twiddle matrix is $\mat{T} = \begin{bmatrix}
    \mat{T}^{(1)} & \mat{T}^{(2)} & \dots & \mat{T}^{(n)}
\end{bmatrix}$
where each $\mat{T}^{(i)}$ has the same dimensions as $\mat{X}^{(i)}$. Therefore
\begin{equation*}
    \mat{Y}'' \gets \mat{T} \odot \mat{Y}' = \begin{bmatrix}
       \mat{T}^{(1)} \odot \mat{F}_k \mat{X}^{(1)} & \mat{T}^{(2)} \odot  \mat{F}_k \mat{X}^{(2)} & \dots & \mat{T}^{(n)} \odot \mat{F}_k \mat{X}^{(n)}
    \end{bmatrix}.
\end{equation*}
Next it right multiplies by $\mat{M}_2$
\begin{equation*}
    \mat{M}_2 =  \texttt{BlockDiagonal}[\mat{F}_{m_1}, \mat{F}_{m_2},\dots, \mat{F}_{m_n}] = \begin{bmatrix}
        \mat{F}_{m_1} & & & \\
        & \mat{F}_{m_2} & & \\
        & & \ddots & \\
        & & & \mat{F}_{m_n}
    \end{bmatrix}.
\end{equation*}
Notice that $\mat{F}_{m_i}$ has dimension $m_i \times m_i$ and $ \mat{T}^{(i)} \odot  \mat{F}_k \mat{X}^{(i)} $ has dimension $k \times m_i$. Therefore,
\begin{equation*}
    \mat{Y}''' \gets \mat{Y}'' \mat{M}_2 =  \begin{bmatrix}
       \left( \mat{T}^{(1)} \odot \mat{F}_k \mat{X}^{(1)} \right) \mat{F}_{m_1} & \left( \mat{T}^{(2)} \odot  \mat{F}_k \mat{X}^{(2)}  \right) \mat{F}_{m_2} & \dots & \left( \mat{T}^{(n)} \odot \mat{F}_k \mat{X}^{(n)} \right) \mat{F}_{m_n}
    \end{bmatrix}.
\end{equation*}
Therefore, it suffices to consider $ \left( \mat{T}^{(i)} \odot \mat{F}_k \mat{X}^{(i)} \right) \mat{F}_{m_i} $ for a single document since they are non-interacting. This expression is Bailey's algorithm for the $L_i' = (k \times m_i)$-point DFT, which is guaranteed to output a matrix 
\begin{equation*}
    \left( \mat{T}^{(i)} \odot \mat{F}_k \mat{X}^{(i)} \right) \mat{F}_{m_i} = \begin{bmatrix}
        \overline{\mat{x}}_0 & \overline{\mat{x}}_k & \dots & \overline{\mat{x}}_{(m_i-1) k} \\
        \overline{\mat{x}}_1 & \overline{\mat{x}}_{k+1} & \dots & \overline{\mat{x}}_{(m_i-1) k + 1} \\
        \vdots & & & \vdots \\
        \overline{\mat{x}}_{k-1} &  \overline{\mat{x}}_{2 k-1} & \dots & \overline{\mat{x}}_{m_i k - 1}
    \end{bmatrix},
\end{equation*}
such that $\overline{\mat{x}}_j = \sum_{\ell = 0}^{L_i' - 1} \omega_{L_i'}^{j \ell} \mat{x}_{\ell}^{(i)}$. Therefore, if we flatten the above expression in column-major order we get $\begin{bmatrix}\overline{\mat{x}}_0 &  \dots & \overline{\mat{x}}_{L_i' - 1} \end{bmatrix}$ which is the $L_i'$-point DFT of $\mat{x}^{(i)}$.
Finally, by definition $\mat{P}_2$ is the map which flattens $\mat{Y}'''$ in column-major order for each document separately and concatenates them so that they maintain the packing order.

Next we prove the computational complexity analysis. 
First we note that the computational complexity of the data pre-processing step is bounded by the size of the outputs it must create. 

Given fixed value $k$, filter length $L_F$, and document lengths
$L_1, \dots, L_n$ (for a fixed batch), {\newalg} increases each
document length by at most $L_i + k$ so
$L_i' \leq 2L_i + k$ and
$m_i \leq 2L_i/k + 1$.

\ignore{Given fixed value $k$ and document lengths $L_1, \dots, L_n$ (for a fixed batch), {\newalg} increases each document length by at most $k$ so $L_i' \leq L_i + k$ and $m_i \leq (L_i/k) + 1$.} 
Let $L_{\text{total}} = L_1 + \dots + L_n$ and $m_{\text{total}} := \sum_{i = 1}^n m_i \leq (2L_{\text{total}}/k) + n$. Ignoring precision, $\mat{T}$ requires at most $O(k m_{\text{total}})$ bits, $\mat{M}_2$ requires at most $O(m_{\text{total}}^2)$ bits, the map $\mat{P}_1$ can be represented in $O(L_{\text{total}} + kn)$ bits, and the map $\mat{P}_2$ can be represented in $O(k m_{\text{total}})$ bits. Therefore the data preprocessing has complexity and memory requirement (per batch and ignoring precision) of $O\left(L_{\text{total}} + kn + (L_{\text{total}}/k)^2 + n^2 + 2L_{\text{total}}n/k \right)$. Picking $k = O(\sqrt{L_{\text{total}}})$ gives a complexity of $O \left( L_{\text{total}} + n^2 + L_{\text{total}}^{1/2} n \right) $. Then as long as $n \leq O(\sqrt{L_{\text{total}}})$ we have that the data preprocessing has cost $O(L_{\text{total}})$.

Next we consider the cost of the main {\newalg} algorithm. Assume a standard dense matrix multiplication with cubic cost. Observe that $\mat{X}$ and $\mat{T}$ have shape $k \times m_{\text{total}}$, $\mat{M}_1$ has shape $k \times k$, and $\mat{M}_2$ has shape $m_{\text{total}} \times m_{\text{total}}$ Therefore the main costs of the algorithm are as follows:
\begin{itemize}
    \item Left multiply by the first $k$-point DFT matrix $\mat{M}_1$ costs $\Theta(k^2 m_{\text{total}})$
    \item Element wise multiply by adaptive twiddle matrix $\mat{T}$ costs $\Theta(k m_{\text{total}})$
    \item Right multiply by the adaptive block-diagonal DFT matrix $\mat{M}_2$ costs $\Theta(k m_{\text{total}}^2)$
\end{itemize}
Therefore the total cost is $\Theta \left(k^2 m_{\text{total}} + k m_{\text{total}}^2  \right) $. Assume we pick $k = L_{\text{total}}^{1/2}$. Using that $m_{\text{total}} \leq (2L_{\text{total}}/k) + n$ yields that the complexity is $\Theta \left( L_{\text{total}}^{3/2} + nL_{\text{total}} + n^2 L_{\text{total}}^{1/2} \right)$. Therefore, whenever the number of packed documents $n$ is bounded by $L_{\text{total}}^{1/2}$ we obtain total cost $\Theta \left( L_{\text{total}}^{3/2} \right)$, which is the same cost as Bailey's algorithm if $n=1$ and no packing were present.

\end{proof}

\section{Implementation Details of {\newalg}} \label{sec:rubi-optimizations}

Algorithm~\ref{alg:rubiconv} only gives a description of the pseudcode, however several additional steps are required to make it performant when implemented on a TPU. These are detailed in this section. 


\paragraph{2D Frequency-Domain Reshaping.} To minimize computational overhead and avoid unnecessary memory operations, our implementation maintains the data in its two-dimensional matrix form following the forward \newalg{} transform. Rather than unpadding and flattening the data back to its original 1D formulation, the element-wise multiplication between the batch and filter representations is performed directly within this 2D frequency domain. However, projecting this product back via the inverse transform (iFFT) requires a structural realignment; the 2D matrix must be reshaped to ensure the data sequences are correctly ordered prior to the inverse operations. The explicit algorithm detailing the construction of this intermediate reshaping matrix is described in Algorithm ~\ref{alg:rubiconv_reshape}.

\begin{algorithm}[H]
\caption{\newalg: Pre-IFFT Reshaping Map Construction}
\label{alg:rubiconv_reshape}
\begin{algorithmic}[1]
\Function{\newalg-IFFT-Reshape-Map}{\text{document\_lengths}, $k$}
    \For{$j=1, \dots, \text{batch\_size}$}
        \State $[L_1, \dots, L_{n_j}] \gets \text{document\_lengths}[j] $
        \State $c_{\text{start}} \gets 1$ \Comment{Initialize global column tracker}
        \For{$i = 1, \dots, n_j$}
            \State $r \gets (k - (L_i \pmod k)) \pmod k$ 
            \State $m_i \gets (L_i + r) / k$ \Comment{Number of columns for document $i$}
            
            \For{$u = 1, \dots, k$} \Comment{Iterate over rows}
                \For{$v = 1, \dots, m_i$} \Comment{Iterate over local columns}
                    \State $c \gets c_{\text{start}} + v - 1$ \Comment{Global source column index}
                    
                    \State \Comment{Calculate 1D frequency index $f$ (0-indexed) from column-major FFT output}
                    \State $f \gets (v - 1) \cdot k + (u - 1)$ 
                    
                    \State \Comment{Map $f$ to row-major destination coordinates $(u', v')$ for the IFFT}
                    \State $u' \gets \lfloor f / m_i \rfloor + 1$ 
                    \State $v' \gets (f \pmod{m_{i}}) + 1$ 
                    \State $c' \gets c_{\text{start}} + v' - 1$ \Comment{Global destination column index}
                    
                    \State \Comment{Store the source coordinate $(u, c)$ at the destination $(u', c')$}
                    \State $\mat{P}_{\text{row}}[j](u', c') \gets u$
                    \State $\mat{P}_{\text{col}}[j](u', c') \gets c$
                \EndFor
            \EndFor
            \State $c_{\text{start}} \gets c_{\text{start}} + m_i$ \Comment{Advance global column tracker}
        \EndFor
    \EndFor
    \State \Return $\mat{P}_{\text{row}}, \mat{P}_{\text{col}}$
\EndFunction
\end{algorithmic}
\end{algorithm}

\paragraph{Precomputation of Adaptive Structures for Multi-Layer Models}
To accelerate model training and eliminate redundant computations in a multi-layer architecture, we optimize the generation of data-dependent components. A key computational challenge is that the twiddle matrix $\mat{T}$ and the second, block-diagonal DFT matrix $\mat{M}_2$ are \textit{adaptive}—they depend on the specific lengths of the documents packed within the input batch. To avoid runtime overhead during training, the required matrix construction and data preparation (including the grid reshaping maps $\mathcal{M}_{\text{row}}, \mathcal{M}_{\text{col}}$) are performed only once as a preprocessing step. Since these structures depend exclusively on the static document boundaries of the batch, they remain invariant across all layers. Sharing these precomputed structures across the network fully amortizes the construction cost, ensuring that the adaptive nature of the transform accelerates, rather than bottlenecks, overall training throughput.



\paragraph{Optimizing Dual Real FFTs via Conjugate Symmetry}
To eliminate the overhead of computing separate forward transforms for a real-valued input batch $\mat{B}$ and filter $\mat{F}$, we pack them into a single complex tensor $\mat{Z} = \mat{B} + i\mat{F}$ and apply a single forward pass of \newalg-FFT (i.e., the core frequency operations from Algorithm 1, bypassing the $\mat{P}_1$ and $\mat{P}_2$ reshaping maps). Because real-valued signals exhibit conjugate symmetry ($\mat{B}(j) = \mat{B}^*(-j)$ and $\mat{F}(j) = \mat{F}^*(-j)$), we isolate the individual transforms directly from $\mat{Z}(j)$ using $\mathcal{O}(N)$ arithmetic, bypassing a second FFT transform entirely:
\begin{align}
    \mat{B}(j) &= \frac{1}{2} \left( \mat{Z}(j) + \mat{Z}^*(-j) \right), \quad
    \mat{F}(j) = \frac{1}{2i} \left( \mat{Z}(j) - \mat{Z}^*(-j) \right)
\end{align}
To compute the inverse transform after point-wise frequency multiplication, we reuse the forward kernel via the standard identity $\text{IFFT}(\mat{X}) = \frac{1}{N} \text{FFT}(\mat{X}^*)^*$. The full procedure is detailed in Algorithm \ref{alg:optimized_rubiconv}.


\begin{algorithm}[H]
\caption{Optimized \newalg{} Convolution (Real-to-Complex Packing)}
\label{alg:optimized_rubiconv}
\begin{algorithmic}[1]
\Function{Optimized-\newalg-Conv}{$\mat{B}, \mat{F}, \text{document\_lengths}, \mat{P}_1, \mat{P}_{\text{ifft\_row}}, \mat{P}_{\text{ifft\_col}}, \mat{P}_2, k$}
    
    \State \Comment{\textbf{1. Padding \& Packing (Implicit Batch Dimension)}}
    \State $\mat{B}_{\text{pad}}, \mat{F}_{\text{pad}} \gets \Call{Pad-To-Grid}{\mat{B}, \mat{F}, \text{document\_lengths}, k}$
    \State $\mat{Z} \gets \mat{B}_{\text{pad}} + i\,\mat{F}_{\text{pad}}$ \Comment{Pack real signals into a single complex tensor}

    \State \Comment{\textbf{2. Forward FFT}}
    \State $\mat{Z}' \gets \Call{\newalg-FFT}{\Re(\mat{Z}),\; \Im(\mat{Z}),\; \text{document\_lengths},\; k}$ \Comment{Returns 2D grid (Alg 1)}

    \State \Comment{\textbf{3. Per-Document Frequency Reversal \& Symmetry Recovery}}
    \State \Comment{For document $d$ with length $N_d = M_d \times C_d$, grid indices $m \in [0, k)$ and $c \in [0, C_d)$:}
    \State $\mat{Z}^{\text{rev}}(<m, c>) \gets \mat{Z}'\!\left((-m) \bmod k,\; (-c - \mathbf{1}_{m>0}) \bmod C_{d(c)}\right)$ \label{line:freq_rev}
    \State $\mat{B}' \gets \tfrac{1}{2}\left(\mat{Z}' + \overline{\mat{Z}^{\text{rev}}}\right)$
    \State $\mat{F}' \gets -\tfrac{i}{2}\left(\mat{Z}' - \overline{\mat{Z}^{\text{rev}}}\right)$

    \State \Comment{\textbf{4. Point-wise Multiplication}}
    \State $\mat{S} \gets \mat{B}' \odot \mat{F}'$

    \State \Comment{\textbf{5. Pre-IFFT Reshaping \& Inverse FFT}}
    \State $\mat{S}_{\text{reshaped}} \gets \mat{S}\!\left(\mat{P}_{\text{ifft\_row}},\; \mat{P}_{\text{ifft\_col}}\right)$
    \State $\mat{Y} \gets \Call{\newalg-FFT}{\Re(\mat{S}_{\text{reshaped}}),\; -\Im(\mat{S}_{\text{reshaped}}),\; \text{document\_lengths},\; k}$
    
    \State \Comment{\textbf{6. Final Scaling \& Unpadding}}
    \State $\mat{Y}_{\text{scaled}}(<m, c>) \gets \Re\!\left(\mat{Y}(<m, c>)\right) \big/ N_{d(c)}$ \Comment{$d(c)$ maps column to document index}
    \State \Return $\mat{P}_2\!\left(\mat{Y}_{\text{scaled}}\right)$
    
\EndFunction
\end{algorithmic}
\end{algorithm}

\paragraph{Arithmetic Intensity and Karatsuba Multiplication}
Because \newalg{} is fundamentally formulated via dense matrix multiplications (GEMMs) of complexity $\mathcal{O}(N^{3/2})$, these operations dominate the execution time. To optimize these internal GEMMs, we apply the Karatsuba algorithm~\cite{karatsuba1962multiplication}. This reduces the required real multiplications from four to three per complex product, substituting them with cheaper additions and lowering the overall arithmetic intensity of the transform. We retain standard complex multiplication for the element-wise twiddle factor application, as its $\mathcal{O}(N)$ complexity does not justify the addition overhead of Karatsuba.

At the macro level, packing the batch and filter into a single forward pass directly halves this dominating computational overhead. More crucially for bandwidth-bound accelerators, executing one transform instead of two halves the number of heavy kernel launches and global memory accesses. Finally, compilers like XLA aggressively fuse the $\mathcal{O}(N)$ frequency-reversal and arithmetic recovery steps alongside the Karatsuba additions, ensuring these auxiliary operations execute entirely within fast on-chip memory with negligible latency.

\paragraph{Optimal Document Padding} 
As described in Section \ref{sec:new_alg}, \newalg{} requires padding all documents to a multiple of $k$. We ablate on the optimal value of $k$ for each sequence length by considering a fixed filter size of $2^{15}$ and a fixed model dimension of $1024$.

Figure \ref{fig:padding_doc_len_ablation} shows that for packed sequences $L_{\text{total}} \leq 2^{15}$ the optimal value is $k=256$ and for sequences $2^{16} \leq L_{\text{total}} \leq 2^{18}$ the optimal value is $k=512$.

\begin{figure}[h!]
    \centering
    \includegraphics[width=0.45\textwidth]{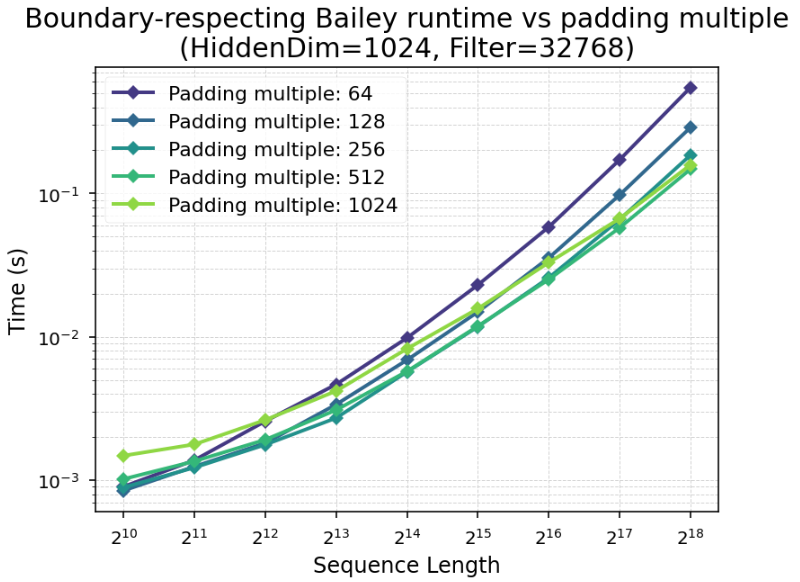}
    \caption{\newalg{} runtime across sequence length for different values of $k$, the greatest common divisors for padded documents.}
    \label{fig:padding_doc_len_ablation}
\end{figure}
\section{Alternative Boundary-Preserving Formulations} \label{sec:app_alternatives}

Section~\ref{sec:alternatives} describes \ctversion, a boundary respecting algorithm for convolutions that adapts the classic iterative Cooley-Tukey FFT algorithm. We include the algorithm illustration in Figure~\ref{fig:ctversion}.

\begin{figure}[h]
    \centering
    \includegraphics[width=\linewidth]
    {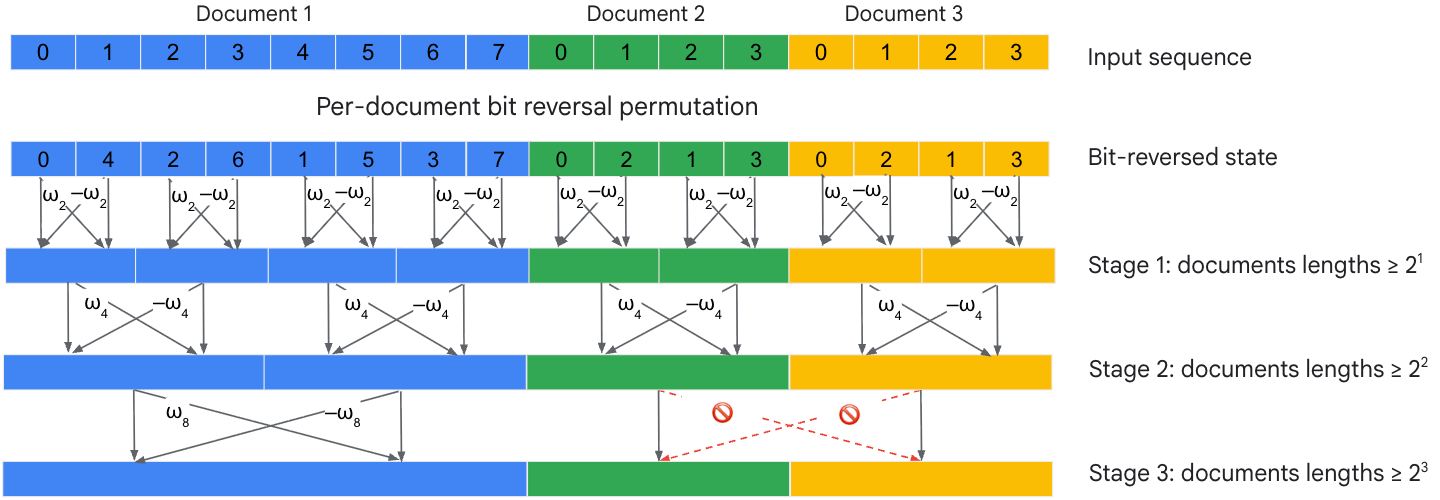}
    \caption{Visual representation of \ctversion. The top layer represents the original document. An arrow without a label represents point-wise multiplication with 1, otherwise it represents multiplication with the label.}
    \label{fig:ctversion}
\end{figure}

\end{document}